\documentclass[11pt]{article}
\usepackage[utf8]{inputenc}
\usepackage[T1]{fontenc}
\usepackage[margin=1in]{geometry}
\usepackage{amsmath,amssymb}
\usepackage{graphicx}
\usepackage{longtable}
\usepackage{caption}
\usepackage{placeins}
\usepackage{enumitem}
\usepackage{natbib}
\usepackage{microtype}
\usepackage{url}
\usepackage{hyperref}

\newtheorem{proposition}{Proposition}
\newtheorem{corollary}{Corollary}

\hypersetup{
  colorlinks=true,
  linkcolor=blue,
  citecolor=blue,
  urlcolor=blue
}

\title{When Known Physics Helps Neural PDE Models:\\
Residual Constraints Out-Regularize Generic Priors for Nonlinear Dynamics}

\author{
Zahra Farazpay \\
Department of Physics \\
Louisiana Tech University \\
Ruston, LA, USA \\
\and
Aniruddha Bora \\
Department of Computer Science \\
Texas State University \\
San Marcos, TX, USA \\
\texttt{aniruddha\_bora@txstate.edu}
}

\date{September 2026}

\begin{document}
\maketitle
\begin{abstract}
Neural PDE surrogates increasingly incorporate structural priors, yet it is often unclear whether their gains arise from physics-specific information or simply from regularization and training choices. We evaluate several such priors under a common protocol against a matched from-scratch neural operator baseline. Our central result is that a known-equation residual consistently outperforms the best generic regularizer at equal tuning budget. At fixed capacity this benefit appears across linear and nonlinear PDEs, but a capacity sweep reveals a sharp distinction: the advantage persists and grows for Burgers, KdV, and Allen--Cahn, while collapsing toward or below parity for linear heat and advection--diffusion. Thus, the durable value of the residual is specific to nonlinear operators. We further falsify a pre-registered hypothesis that the benefit is activated only by data sparsity: the residual remains advantageous even under full supervision. Its usefulness does, however, have a clear boundary. Under grid under-resolution, nonlinear coarse fields no longer satisfy the naive governing-equation residual, and enforcing it becomes actively harmful. In contrast, cross-family pretraining and in-context conditioning fail to outperform the strong from-scratch baseline in the regime studied. Together, these results identify when known physics provides non-redundant information to neural PDE models, when it does not, and when enforcing it introduces bias.
\end{abstract}
\section{Introduction}
Neural surrogates for partial differential equations (PDEs) increasingly incorporate known structure through governing-equation penalties, cross-family pretraining or conditioning, and discretization-invariant operator architectures \citep{raissi2019,li2024pino,subramanian2023,herde2024,mccabe2024,yang2023,li2020fno}. The usual premise is that such structure improves generalization. Yet structural priors are often compared with weak references---for example, an unregularized model or a from-scratch model whose own inductive biases are not tuned---making it difficult to separate the value of the specific prior from the generic value of regularization or pretraining. This broader baseline problem has also been documented in evaluations of learned PDE solvers \citep{mcgreivy2024}.

We study a narrower question: \emph{when does known structure improve neural PDE prediction beyond a strong baseline?} Our baseline is a from-scratch neural operator with matched architecture and validation-selected generic regularization. The headline residual and capacity comparisons use one-step training in both the physics and generic arms; multi-step pushforward is evaluated separately as a physics-free rollout-stabilization control and is used in the transfer and generalization experiments as specified in Appendix~\ref{app:hparams}. Every structural prior is evaluated under the same rollout metric and, crucially, the known-equation residual is compared with the \emph{best generic regularizer at equal tuning budget}, rather than with an unregularized network. We use heat and advection--diffusion as linear controls and Burgers, KdV, and Allen--Cahn as nonlinear families.

Our main finding is that the known-equation residual is the one structural prior that reliably clears this bar, but its durable advantage is sharply conditioned. At fixed width, the residual improves held-out prediction for all five families. As capacity grows, however, the advantage persists and grows only for the nonlinear families; for the linear controls it collapses toward parity or below it. We therefore describe the result as a \emph{capacity-robust advantage for nonlinear operators}, rather than as a fixed-width ordering by nonlinearity. The residual also remains beneficial at full supervision, falsifying our pre-registered hypothesis that its value is activated only by data sparsity.

The positive result has an equally sharp failure mode. When nonlinear dynamics are under-resolved, the true coarse field no longer satisfies the naive coarse-grid equation because unresolved scales induce a closure term. Enforcing the resulting inaccurate residual reverses the benefit and can make the physics-penalized model substantially worse than the generic baseline. This identifies resolvable evaluation of the governing equation as a necessary condition for using the residual as a training constraint.

Finally, we place this result against cross-family transfer. In the low-source-diversity, modest-target-data regime studied here, pretraining and in-context conditioning do not beat the strong from-scratch baseline once more than a single target trajectory is available. Additional relevance-gradient controls are consistent with optimization-basin compatibility being an important determinant of transfer in this regime, rather than source relevance alone; we keep this interpretation deliberately scoped and report the full comparison in the main results.

\paragraph{Contributions.} We make three main contributions: (i) a controlled comparison showing that a known-equation residual out-regularizes equally tuned generic alternatives, with a capacity-robust advantage specific to nonlinear PDEs; (ii) a mechanistic resolution boundary showing that the benefit reverses when nonlinear coarse fields violate the residual through spectral-truncation closure; and (iii) a controlled negative result showing that cross-family transfer does not outperform the strong from-scratch baseline in the regime studied. We additionally use pushforward controls, a relevance gradient for transfer, and resolution-versus-geometry tests to delineate which structural biases survive stronger baselines and which do not. Detailed architecture settings, pre-registration records, derivations, replication tables, and per-condition statistics are provided in the appendix.
\section{Methods}
\subsection{PDE benchmark and neural operator}
We study five one-dimensional periodic PDE families: heat and advection--diffusion as linear controls, and viscous Burgers, KdV, and Allen--Cahn as nonlinear systems,
\begin{align}
\text{heat:} \quad & u_t = \nu\,u_{xx}, &
\text{advection--diffusion:} \quad & u_t = -c\,u_x + \nu\,u_{xx}, \nonumber\\
\text{Burgers:} \quad & u_t = -u\,u_x + \nu\,u_{xx}, &
\text{KdV:} \quad & u_t = -u\,u_x - \delta\,u_{xxx}, \nonumber\\
\text{Allen--Cahn:} \quad & u_t = \varepsilon\,u_{xx} + u - u^3. & &
\label{eq:pde_families}
\end{align}
Reference trajectories are generated with Fourier pseudo-spectral spatial discretization and ETDRK4 time integration, with dealiasing for nonlinear terms. Initial conditions are smooth random fields with controlled Fourier content. Unless resolution is itself the experimental variable, the modeling grid is chosen so that the reference trajectory satisfies the discrete governing equation to high accuracy.

We use two coefficient regimes. In the headline and capacity experiments, coefficients are fixed within each PDE family, so generalization is across initial conditions at a fixed operator. In the transfer and relevance experiments, coefficients are sampled trajectory-wise, so evaluation requires generalization across both initial conditions and PDE parameters. Exact coefficients, sampling ranges, time steps, and trajectory counts are given in Appendix~\ref{app:families}.

Our base surrogate is a residual-form one-dimensional Fourier Neural Operator (FNO) \citep{li2020fno}. The headline residual and capacity experiments use matched one-step training ($K=1$) in both the physics and generic arms. Multi-step pushforward supervision \citep{brandstetter2022} is evaluated separately as a rollout-stabilization control and is used in the transfer, resolution-invariance, and geometry experiments, as detailed in Appendix~\ref{app:hparams}. All models are optimized with Adam, warmup/cosine decay, and gradient clipping, and are evaluated by held-out rollout relative $L_2$ error at a fixed horizon. Additional U-Net and Transformer backbones are used for transfer controls.

\subsection{Physics residual and controlled baselines}
For an input snapshot $u_0$ and predicted next state $u_1$, with governing equation $u_t=F(u)$, we evaluate the pointwise Crank--Nicolson residual
\begin{equation}
 r(u_0,u_1)
 = \frac{u_1-u_0}{\Delta t}
 - \tfrac12\big(F(u_0)+F(u_1)\big),
 \qquad
 \mathcal{L}_{\mathrm{phys}}
 = \lambda\,\big\langle r^2\big\rangle.
\label{eq:physics_residual}
\end{equation}
The residual is evaluated at all $N$ collocation points, including locations not used by the masked data term. The physics arm therefore minimizes the observed-data loss together with this equation constraint,
\begin{equation}
\mathcal{L}
= \big\langle (u_1-\hat u_1)^2\big\rangle_{\mathrm{obs}}
+ \mathcal{L}_{\mathrm{phys}}.
\label{eq:training_objective}
\end{equation}
To distinguish physics-specific information from generic regularization, the physics arm is not compared with an unregularized model. Instead, for each family and data regime we tune the physics coefficient and, with a matched search budget, tune generic alternatives consisting of data-only training, weight decay, and input-noise augmentation. Configuration selection uses a seed-disjoint validation ensemble; the test set is never used for selection. We report the test error of the best validation-selected configuration in each arm. Linear families serve as controls for whether any durable gain depends on operator nonlinearity. The resolved-grid validity of Eq.~\ref{eq:physics_residual} and its failure under nonlinear coarse-graining are developed in Appendix~\ref{app:residual}.

\subsection{Resolution, capacity, and transfer tests}
The capacity sweep varies model width over roughly an order of magnitude at full supervision. The resolution sweep generates reference dynamics on a fine grid and progressively coarsens the grid on which the model and residual are evaluated. For transfer, source families are used either for pretrain--finetune initialization or as support trajectories for FiLM/attention-based conditioning; every transfer arm is compared with a from-scratch model trained on exactly the same target trajectories.

\subsection{Statistical protocol and pre-registration}
Unless stated otherwise, comparisons use five random seeds controlling initialization, sampling, and training stochasticity. We bootstrap the paired error difference or error ratio across seeds/repetitions and regard an effect as resolved when the corresponding $95\%$ interval excludes zero (difference) or one (ratio). Before the deciding sparsity experiment we pre-registered the hypothesis that the residual advantage would vanish at dense supervision and cross over below a bandwidth-set threshold; the data falsify that hypothesis. The complete pre-registration record is in Appendix~\ref{app:prereg}.
\section{Results}
\subsection{Pushforward training strengthens rollout stability}

\begin{figure}[tb]
\centering
\includegraphics[width=0.62\linewidth]{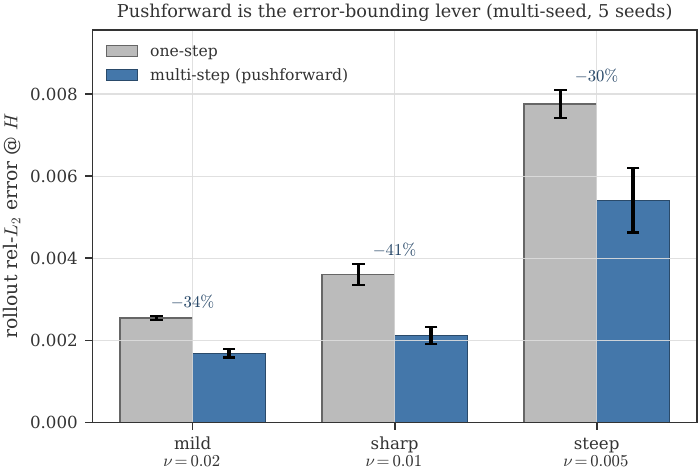}
\caption{\textbf{Pushforward, not physics, strengthens rollout stability.} One-step versus multi-step (pushforward) training across three viscosity regimes (5 seeds); pushforward reduces rollout error by 30--41\%. This is a physics-free baseline-strengthening control; the headline residual and capacity comparisons remain matched one-step experiments.}
\label{fig:pushforward}
\end{figure}

Figure~\ref{fig:pushforward} isolates a physics-free way to strengthen rollout stability. Across the viscosity ladder (mild/sharp/steep, five seeds), multi-step pushforward training reduces rollout error relative to one-step training by 34\%, 41\%, and 30\%, with confidence intervals excluding zero. We treat this as a control rather than a contribution: it shows that rollout-aware supervision can materially improve a from-scratch model without injecting equation information. The headline residual and capacity experiments in Fig.~\ref{fig:headline}, however, keep both the physics and generic arms on the same one-step protocol so that their difference isolates the residual constraint from the generic regularization alternatives.

\subsection{A known-equation residual out-regularizes generic priors}
\begin{figure*}[t]
\centering
\includegraphics[width=0.96\textwidth]{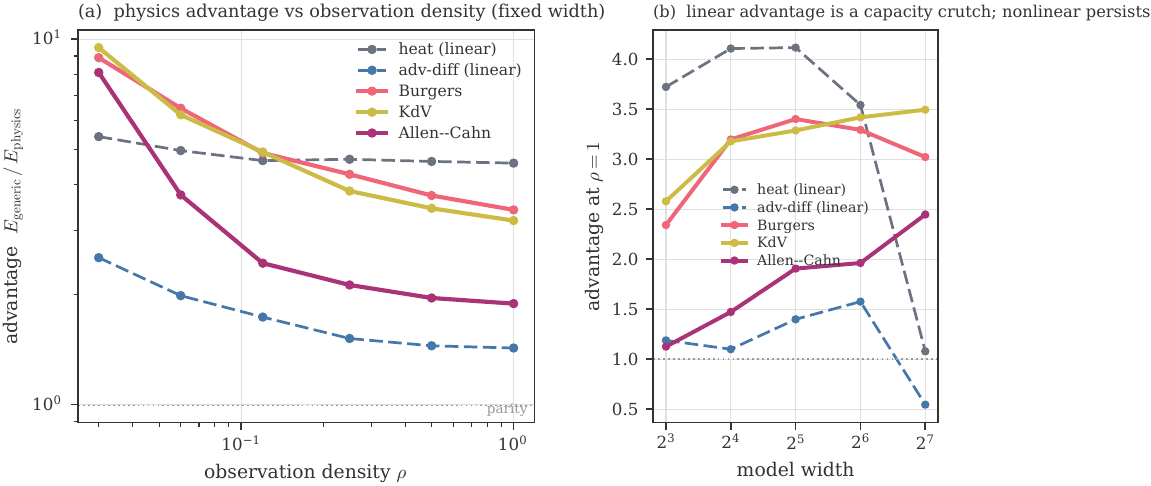}
\caption{\textbf{Known-equation residuals provide a capacity-robust advantage for nonlinear operators.} \textbf{(a)} Advantage $E_{\mathrm{generic}}/E_{\mathrm{physics}}$ versus observation density $\rho$ at fixed width; values above one favor the physics residual. The residual wins across all tested densities, including full supervision. \textbf{(b)} At full supervision, increasing width separates linear and nonlinear behavior: the nonlinear families retain a clear advantage while the linear controls approach or cross parity.}
\label{fig:headline}
\end{figure*}

Figure~\ref{fig:headline} compares the validation-selected physics arm with the best generic regularizer at equal tuning budget. At fixed width, the residual improves held-out prediction at every observation density. At full supervision, Burgers improves from $5.8\times10^{-5}$ to $1.7\times10^{-5}$ ($3.4\times$), KdV by $3.2\times$, and Allen--Cahn by $1.9\times$. The advantage grows as data become sparse, but it never switches on at a sparsity threshold: it is already present at $\rho=1$.

The fixed-width ratio itself does \emph{not} order families by nonlinearity: heat exhibits a large ratio because the residual essentially supplies its linear map, while absolute errors are already $O(10^{-6})$. The decisive experiment is the capacity sweep. At the largest width, the nonlinear families retain advantages of $3.1\times$ for Burgers (CI $[2.95,3.20]$), $3.5\times$ for KdV ($[3.30,3.75]$), and $2.5\times$ for Allen--Cahn ($[2.17,2.73]$). By contrast, heat falls to $1.95\times$ with an interval including parity ($[0.76,3.03]$), and advection--diffusion falls below parity to $0.64\times$ ($[0.49,0.84]$). The robust statement is therefore a capacity-persistent nonlinear/linear dissociation, not simply that larger fixed-width ratios imply greater nonlinearity.

Ground-truth residual checks verify that the resolved-grid comparison is well posed: the reference residual is near machine precision for the linear families and small for the nonlinear families. The pre-registered sparsity-gating hypothesis is falsified because no dense-supervision inert regime appears; the monotone sparsity trend is real, but the residual behaves as a better inductive bias rather than a constraint whose value turns on only when labels are scarce.

\begin{figure}[t]
\centering
\includegraphics[width=0.92\linewidth]{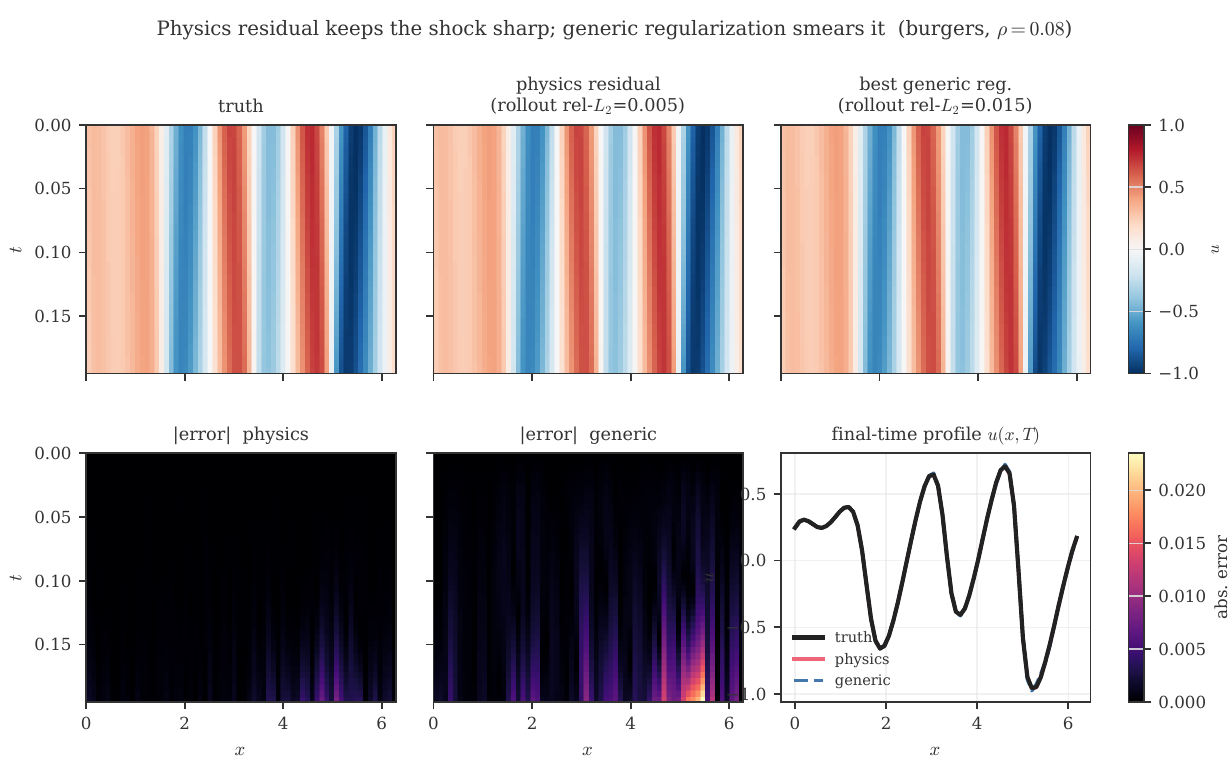}
\caption{\textbf{Illustrative Burgers rollout at sparse supervision ($\rho=0.08$).} The physics-penalized rollout preserves the shock more accurately than the best generic-regularized model (rollout error $5.2\times10^{-3}$ versus $1.5\times10^{-2}$). Quantitative multi-seed results are in Fig.~\ref{fig:headline}.}
\label{fig:headline_soln}
\end{figure}

\subsection{Non-redundant operator content predicts the capacity-persistent gain}
\label{sec:redundancy}

The closure analysis characterizes \emph{when} the residual is a valid constraint (resolution). A second operator property characterizes \emph{how much durable value} a valid residual carries as model capacity grows. Let $G_{\Delta t}\colon u\mapsto u(t{+}\Delta t)$ be the one-step solution operator and $\delta(u)=G_{\Delta t}(u)-u$ the update. Define the \emph{non-redundant content}
\[
\rho_{\mathrm{nr}}(F)=\frac{\|\delta - A^\star u\|}{\|\delta\|},\qquad A^\star=\arg\min_{A\ \text{linear, transl.-invariant}}\|\delta-Au\|,
\]
the relative norm of the update not representable by the best linear operator; equivalently, the optimum is the per-wavenumber Wiener filter $A^\star(k)=\langle\hat\delta(k)\,\hat u(k)^*\rangle/\langle|\hat u(k)|^2\rangle$.

\begin{proposition}[Non-redundancy dichotomy]\label{prop:redundancy}
$\rho_{\mathrm{nr}}(F)=0$ if and only if $F$ is linear. For linear $F$, $G_{\Delta t}=e^{L\Delta t}$ acts as a Fourier multiplier, so $\hat\delta(k)=(e^{\hat L(k)\Delta t}-1)\hat u(k)=M(k)\hat u(k)$; the optimal filter recovers $A^\star=M$ exactly and the residual vanishes. For nonlinear $F$ the flux couples wavenumbers, no linear operator represents $\delta$, and $\rho_{\mathrm{nr}}>0$.
\end{proposition}

\smallskip\noindent\emph{Empirical prediction.} The motivation is the indirect-data reading: at large capacity a model recovers any linear map from data unaided, so the residual's \emph{durable} value is its non-redundant content $\rho_{\mathrm{nr}}$. Across the five families, $\rho_{\mathrm{nr}}$ predicts the capacity-robust advantage $E_{\mathrm{generic}}/E_{\mathrm{physics}}$ at the largest width (Pearson $0.90$; Fig.~\ref{fig:redundancy}): the two families with $\rho_{\mathrm{nr}}=0$ (heat, advection--diffusion) are exactly those whose advantage sits at or below parity, while the three with $\rho_{\mathrm{nr}}>0$ retain a durable advantage. This identifies \emph{operator linearity}, not smoothness, as the operative variable: advection--diffusion is linear but \emph{not} smoothing (it preserves the spectrum), yet $\rho_{\mathrm{nr}}=0$ correctly predicts its advantage falls below parity, a case the smoothness reading handles awkwardly. We report this as a measured predictor, not a magnitude law. With five families the correlation is suggestive rather than conclusive, and $\rho_{\mathrm{nr}}$ resolves the linear/nonlinear dissociation and the weakest nonlinear family (Allen-Cahn) but not the Burgers-versus-KdV ordering; a stiffness-weighted variant ($|\hat L(k)|^2$-weighting, motivated a priori by the harder-to-learn stiff modes) was tested and \emph{did not} improve the fit, so that residual ordering is not a linear-stiffness effect.

\begin{figure}[tb]
\centering
\includegraphics[width=0.66\linewidth]{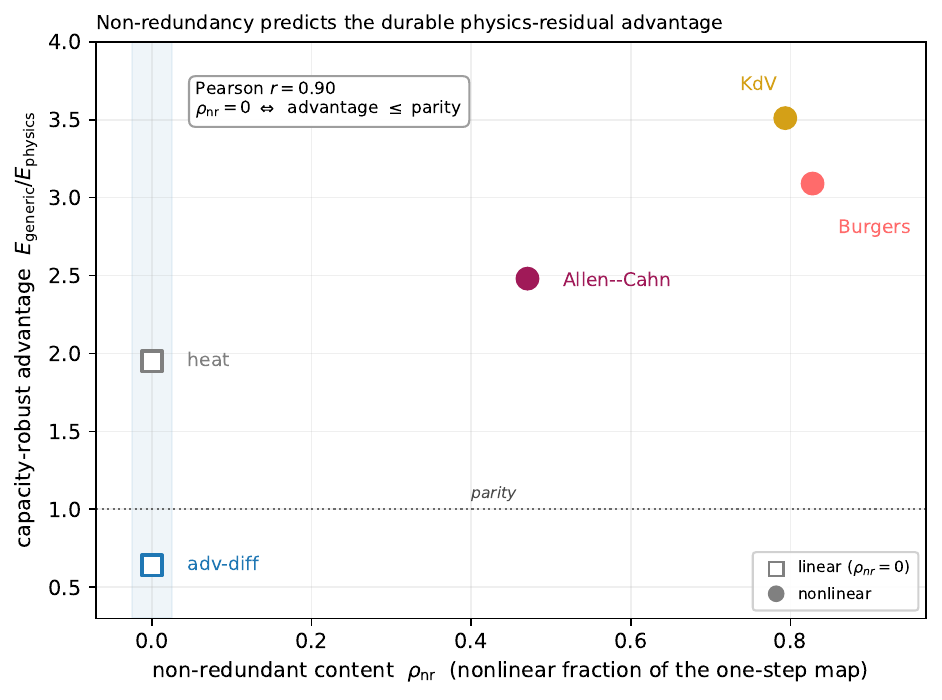}
\caption{\textbf{Non-redundancy predicts the durable physics-residual advantage (Prop.~\ref{prop:redundancy}).} Capacity-robust advantage $E_{\mathrm{generic}}/E_{\mathrm{physics}}$ at the largest width versus the non-redundant content $\rho_{\mathrm{nr}}$, the fraction of the one-step map no linear operator captures. The two linear families (open squares) sit at $\rho_{\mathrm{nr}}=0$ and straddle parity (heat just above, advection--diffusion below); the three nonlinear families lie at $\rho_{\mathrm{nr}}>0$ with durable advantages (Pearson $0.90$). $\rho_{\mathrm{nr}}=0$ coincides with the advantage collapsing to or below parity, identifying linearity rather than smoothness as the operative variable.}
\label{fig:redundancy}
\end{figure}

\subsection{The benefit reverses under nonlinear under-resolution}
When the modeling grid is progressively coarsened, the nonlinear advantage reverses. Burgers and KdV reach physics-to-generic advantage ratios of roughly $0.16$--$0.17$ in the most under-resolved cases, meaning the residual penalty is about five to six times worse than the generic baseline. In the same sweep, the residual of the \emph{true} coarse field grows from roughly $10^{-7}$ to $10^{-1}$. The linear heat control remains stable because Fourier truncation commutes with its linear operator.

The mechanism is the spectral-truncation closure term. Let $P_N$ denote projection onto resolved modes, $Q_N=I-P_N$, and define
\[
\tau_N(u)=F_N(P_Nu)-P_NF(u).
\]
For linear $F$, $\tau_N\equiv0$. For the quadratic Burgers flux,
\[
\tau_N(u)=-\tfrac12\partial_xP_N\!\left[(Q_Nu)(u+P_Nu)\right],
\qquad
\|\tau_N(u)\|_{H^{-1}}\le C_s\|Q_Nu\|_{L^2}\|u\|_{H^s}.
\]
Thus the true projected trajectory violates the naive coarse residual by an amount controlled by unresolved energy. Enforcing that residual introduces a bias that cannot be removed by more labels or model capacity; the remedy is to resolve the dynamics or model the missing closure.

\begin{figure}[t]
\centering
\includegraphics[width=0.96\linewidth]{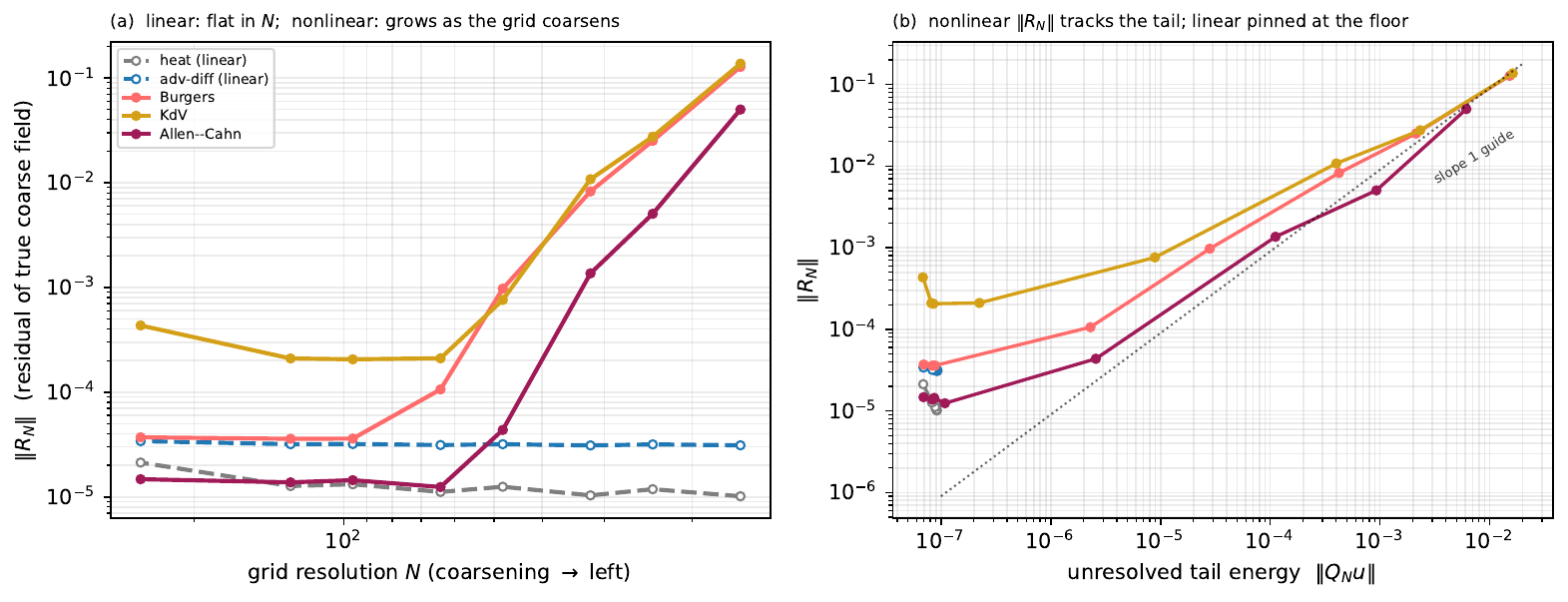}
\caption{\textbf{Resolution boundary.} The true coarse-field residual stays at the discretization floor for linear families but grows rapidly under coarsening for nonlinear families, tracking unresolved tail energy. This is the failure mode of the naive physics penalty under under-resolution.}
\label{fig:closure}
\end{figure}

\subsection{Cross-family transfer does not beat the strong baseline}
\begin{figure}[t]
\centering
\includegraphics[width=0.72\linewidth]{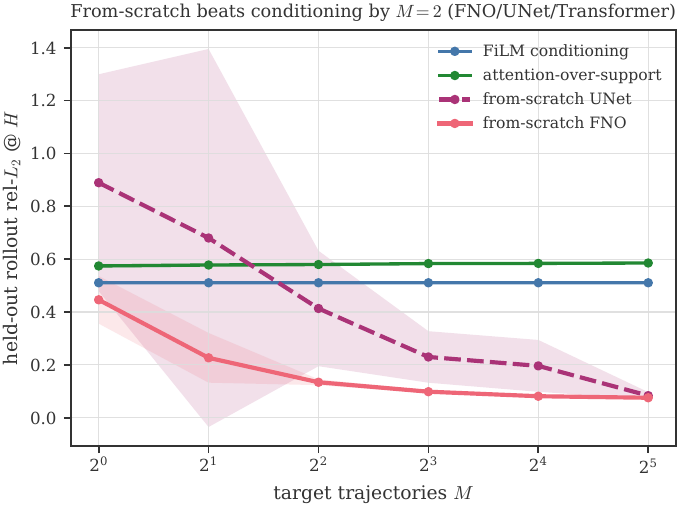}
\caption{\textbf{Cross-family transfer versus from-scratch training.} In-context FiLM and attention-over-support remain nearly flat as target trajectories increase, while from-scratch FNO and U-Net models improve rapidly and overtake them by $M=2$.}
\label{fig:conditioning}
\end{figure}

Cross-family conditioning and pretraining provide a useful contrast to the residual result. In the tested low-source-diversity regime, FiLM and attention-over-support do not improve with additional target support in the way the from-scratch models do; by $M=2$, the from-scratch FNO and U-Net are already below the conditioning arms (Fig.~\ref{fig:conditioning}). The Transformer meta-training loss converges, so the result is not explained simply by a visibly untrained conditioner. A separate relevance-gradient experiment, reproduced across Burgers, KdV, and Allen--Cahn and reported in Appendix~\ref{app:relmulti}, finds that related-PDE pretraining can produce negative transfer while a rich non-physical operator is approximately neutral. These controls are consistent with optimization-basin compatibility being an important determinant of transfer in this regime, but they do not establish a universal law about pretraining at larger source diversity or foundation-model scale.

Figure~\ref{fig:conditioning} reports the cross-family conditioning experiment: a model is trained on a set of source families and evaluated on a held-out target, either by pretrain-then-finetune or by in-context conditioning. Both in-context mechanisms FiLM modulation and Transformer attention-over-support produce held-out error curves that are \emph{flat} in the number of target trajectories $M$, and a from-scratch operator trained on the same M trajectories dives below them by $M=2$. The Transformer's meta-training converged (loss 12.7 $\to$ $2.6\times10^{-3}$, so the conditioner is not undertrained) yet it lands no better than the simpler FiLM arm, and the from-scratch FNO beats the U-Net co-baseline almost everywhere, closing the ``weak baseline'' and ``under-capacity'' objections \citep{mcgreivy2024}. One plausible explanation is the low source diversity: with four source families the conditioner learns to reproduce source dynamics, the support set barely moves the held-out prediction (hence the flat curves), while the from-scratch model uses every target trajectory.

\subsection{Source relevance alone does not predict positive transfer}

\begin{figure}[tb]
\centering
\includegraphics[width=\linewidth]{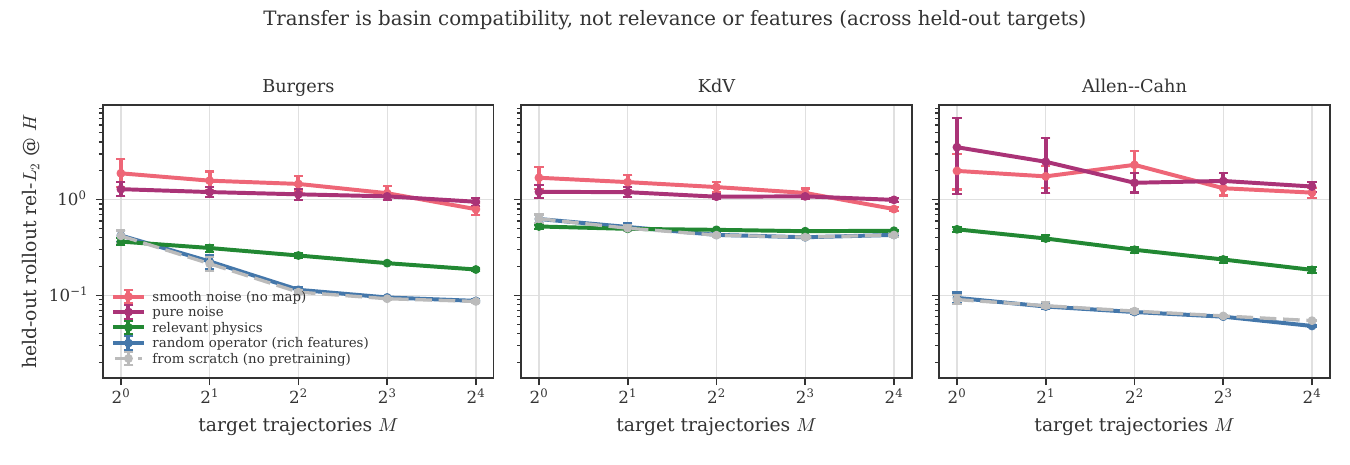}
\caption{\textbf{Transfer is more consistent with basin compatibility than source relevance across three held-out targets.} For each held-out target (Burgers, KdV, Allen--Cahn), held-out rollout error vs.\ adaptation trajectories $M$ for pretraining sources on a relevance gradient. Pretraining on relevant physics (green) fails to beat from-scratch (grey dashed) in the data-adequate regime and is a clear liability on all three targets; a rich non-physical operator (blue) is inert to marginally helpful; structureless pretraining (noise) is destructive. Related-PDE pretraining is worse than from-scratch in the data-adequate regime for all three targets.}
\label{fig:relevance}
\end{figure}

The negative-transfer result above shows that pretraining on related PDE families does not help a held-out target relative to training from scratch. It does not, by itself, identify \emph{why}. Two competing explanations are common in the foundation-model literature: that transfer is governed by the \emph{relevance} of the pretraining data (the ``right physics'' premise), or that it is governed by the \emph{features and capacity} a model acquires from being trained on anything sufficiently rich (the ``pretraining provides a useful initialization'' premise). These explanations make opposite predictions, and we separate them with a controlled relevance gradient.

We fix the target family (Burgers, held out) and pretrain one operator body per source condition, holding architecture, capacity, and pretraining budget fixed across conditions, then finetune on M target trajectories with multi-step (pushforward) supervision and measure held-out rollout error. The source conditions span a gradient from relevant physics to pure noise: \textbf{pde\_related}, real one-step pairs from the other PDE families (the ``right data''); \textbf{struct\_random}, spectrum-matched smooth inputs mapped through a \emph{fixed random smooth operator} (a consistent, learnable map carrying rich features but no physics); \textbf{noise\_map}, smooth inputs mapped to \emph{independent} smooth outputs (no consistent map to learn); and \textbf{pure\_noise}, white-noise inputs and outputs (no structure). A from-scratch model (random initialization, no pretraining) is the reference.

The result challenges both simple explanations. \textbf{In this cross-family regime, source relevance does not predict positive transfer and can instead produce negative transfer.} For all $M \geq 2$, pretraining on relevant physics produces clean \emph{negative} transfer: at $M=16$ the pde\_related model reaches a held-out error of 0.186 while the from-scratch reference reaches 0.094, a gap whose confidence intervals are cleanly separated. Relevant pretraining ties the from-scratch baseline only at the single-trajectory extreme ($M=1$: 0.365 vs 0.419, overlapping intervals), and is worse everywhere thereafter. There is no regime in which it wins.

\textbf{Feature richness is not the driver either.} The \textbf{struct\_random} condition---a model pretrained on a consistent but non-physical smooth operator that acquires rich, learnable features---is statistically indistinguishable from random initialization at every M (e.g., $M=16$: 0.088 vs 0.094, overlapping intervals throughout). Learning \emph{some} operator transfers nothing over not pretraining at all. And when the learnable map is removed (noise\_map, pure\_noise), pretraining is not merely unhelpful but actively destructive, leaving the model an order of magnitude worse and never recovering within the target-data range tested ($M=16$: 0.79 and 0.95 respectively). ``Enough features'' therefore buys nothing when the features are structured-but-irrelevant, and poisons the model when they are unstructured.

We read these controls as evidence that, in this regime, transfer is more consistent with \emph{optimization-basin compatibility} than with source relevance or feature richness alone: whether pretraining places the model in a parameter region that target finetuning can readily adapt from. Related-but-distinct physics can provide an initialization that is difficult to adapt to Burgers' shock-steepening dynamics, whereas a random smooth operator is approximately neutral. A control that triples the finetuning budget (450 versus 150 epochs at $M=16$) leaves the pde\_related error unchanged at 0.186 while from\_scratch remains at 0.094, suggesting that the observed negative transfer is not merely a consequence of too few finetuning epochs. Structureless pretraining is substantially worse and does not recover within the target-data range tested. These observations motivate basin compatibility as a plausible explanation for the transfer pattern, rather than establishing it as a universal mechanism. This dissection is bounded to the regime of low source-family diversity (four source families) and modest target data; it does not establish that data relevance never matters at the scale and diversity of PDE foundation models, and it is consistent with concurrent reports that from-scratch training beats pretrained models out of distribution \citep{pita2025}.

\textbf{The pattern replicates across held-out targets.} Repeating the full gradient for KdV and Allen-Cahn (Fig.~\ref{fig:relevance}; Appendix~\ref{app:relmulti}) reproduces the central refutation on both: the most relevant source is again the worst. At the data-adequate end ($M=16$), pretraining on relevant physics is significantly worse than from-scratch for KdV ($0.47$ vs $0.43$) and markedly worse for Allen-Cahn ($0.185$ vs $0.055$, a $3.4\times$ penalty), with separated intervals in both cases; the rich non-physical operator is statistically inseparable from from-scratch for KdV and gives Allen-Cahn a small but significant gain ($0.048$ vs $0.055$); structureless pretraining is destructive throughout. Two boundary details temper the picture without overturning it: for KdV, relevant-physics pretraining provides a transient head start at the single sparsest point ($M=1$: $0.52$ vs $0.63$) that reverses by $M=4$, and for Allen Cahn the rich-operator basin is mildly compatible rather than perfectly neutral. Across these targets, source relevance alone does not predict the outcome; the results remain more consistent with basin compatibility than with relevance or feature richness by themselves.

\subsection{Resolution invariance does not imply geometry generalization}

\begin{figure}[tb]
\centering
\includegraphics[width=0.52\linewidth]{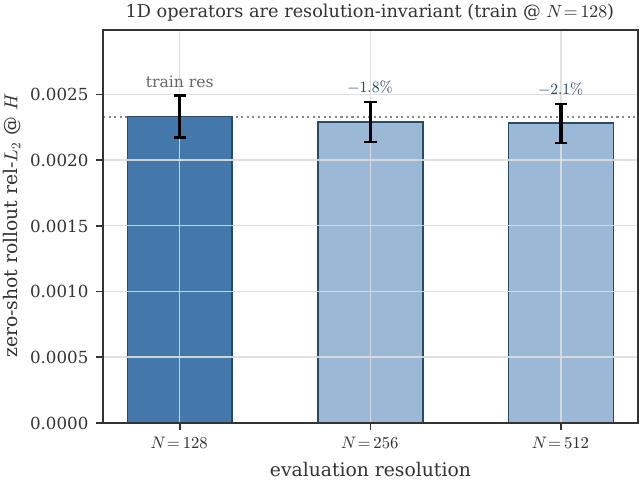}
\caption{\textbf{1D operators are resolution-invariant (scoped positive).} Trained at $N=128$, evaluated zero-shot at finer resolutions; rollout error is unchanged within $\sim$2\% across five seeds.}
\label{fig:resolution}
\end{figure}

\begin{figure}[tb]
\centering
\includegraphics[width=0.48\linewidth]{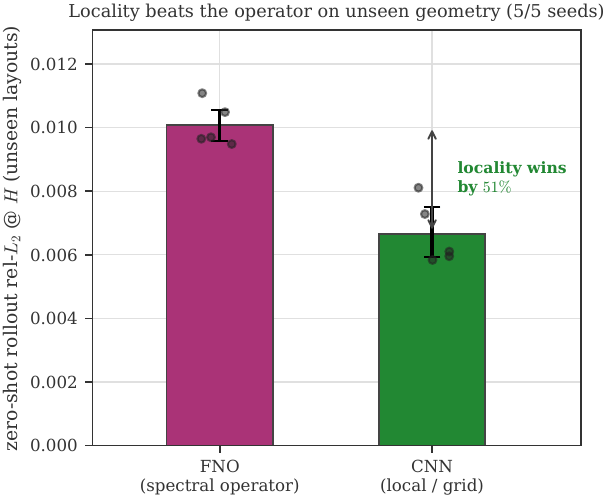}
\caption{\textbf{On unseen geometry, locality beats the operator (scoped tension).} Zero-shot rollout error on unseen obstacle layouts: a local CNN outperforms the FNO by $\sim$50\% across five seeds with separated intervals. Locality and spectral structure trade off on different generalization axes.}
\label{fig:geometry}
\end{figure}

\begin{figure*}[tb]
\centering
\includegraphics[width=\textwidth]{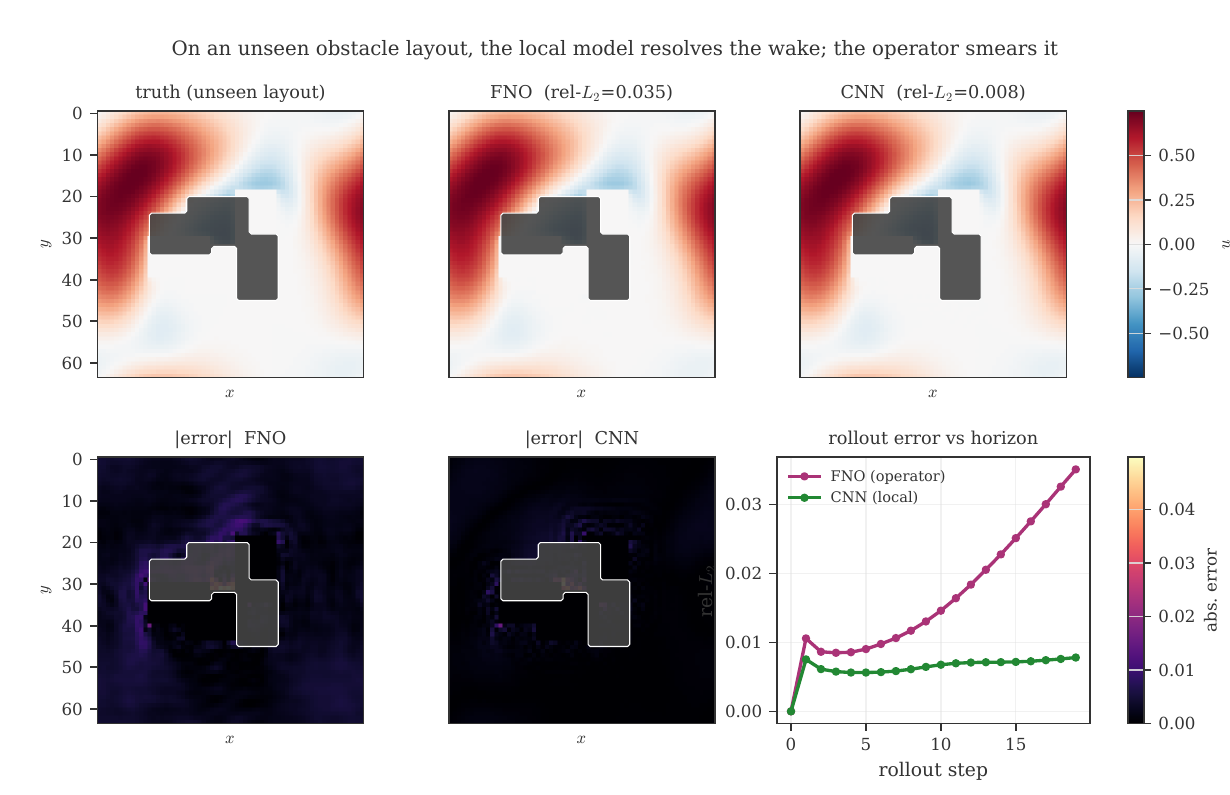}
\caption{\textbf{Geometry, in the field.} Truth, FNO, and CNN rollouts on an unseen obstacle layout (obstacles overlaid), with error fields and rollout error versus horizon. The local model resolves the wake; the operator smears it and its error grows over the rollout. Illustrative single case; the quantitative result is Fig.~\ref{fig:geometry}.}
\label{fig:geometry_soln}
\end{figure*}

Two structural properties survive as genuine but regime-bounded positives. First, one-dimensional operators are resolution-invariant (Fig.~\ref{fig:resolution}): trained at $N=128$ and evaluated zero-shot at $N=256$ and 512, rollout error changes by only -1.2\% to -2.9\% across five seeds---essentially flat, and if anything slightly better at higher resolution. This is a genuine but bounded positive: resolution invariance in this one-dimensional setting does not imply robustness to changes in geometry or spatial structure, and it does not imply operators are the right bias for every generalization axis.

Indeed, on geometry generalization the opposite holds (Fig.~\ref{fig:geometry}; a representative field is shown in Fig.~\ref{fig:geometry_soln}). On a 2D advection--diffusion system with per-trajectory random obstacle layouts, evaluated zero-shot on unseen layouts, a local convolutional model beats the spectral operator in all five seeds: FNO $1.01\times10^{-2}$ versus CNN $6.66\times10^{-3}$, a ~53\% advantage (bootstrap CI [38\%, 68\%]) with separated confidence intervals. Locality and spectral structure therefore trade off, each winning on a different generalization axis: the operator on resolution, the local model on geometry. This trade-off is itself task-dependent: in a harder geometry regime the margin narrows and can reverse, so we do not claim that a single architecture dominates. The scoped conclusion is that structural priors are neither uniformly helpful nor uniformly inert, and the contribution is in mapping which holds when.

\FloatBarrier
\section{Discussion}
The experiments support an information-channel interpretation of the known-equation residual, while also exposing its limits. First, the benefit persists at full supervision, so it is not merely a substitute for missing labels. Second, the benefit that survives increasing model capacity is specific to nonlinear operators: a sufficiently expressive model can recover the linear controls from data, whereas the nonlinear families retain a non-redundant gain from the equation. Third, that channel is useful only when the constraint itself is valid. Under nonlinear under-resolution, the residual encodes the wrong coarse-grid relation because unresolved modes induce closure terms, turning structural information into structural bias.

This interpretation is consistent with recent statistical-learning views of physics penalties as indirect data rather than ordinary regularization \citep{barajas2026}, but our evidence is empirical and prediction-side: we do not claim to establish a general loss-landscape mechanism.

The auxiliary experiments sharpen the scope of this conclusion. Multi-step pushforward training provides a strong, physics-free rollout-stabilization control. Because the headline residual and capacity comparisons use matched one-step training, their gains should be interpreted specifically as residual-versus-generic-regularization effects under that protocol, rather than as gains over a pushforward-trained baseline. Multi-step supervision is used in the transfer and generalization experiments as specified in Appendix~\ref{app:hparams}. Cross-family conditioning and pretraining do not clear the corresponding from-scratch baselines in the tested low-diversity regime, and the relevance-gradient controls show that semantically related source physics is not sufficient for positive transfer. Finally, the resolution and geometry experiments separate two forms of structural generalization: the FNO is effectively resolution-invariant in the one-dimensional test, yet a local CNN is better on unseen obstacle geometry. These results argue against treating ``more structure'' as a scalar quantity; the useful bias depends on the generalization axis. The linear/nonlinear capacity dissociation and the resolution reversal provide two concrete conditions that any broader theory must accommodate. Our results are also deliberately scoped to one-dimensional PDE families, fixed coefficients within each family, resolved grids except in the resolution study, and low source-family diversity. Whether the same boundaries persist at large multi-physics foundation-model scale remains open.

\section{Related Work}
Physics-informed neural networks and neural operators encode governing equations through residual penalties or operator-level constraints \citep{raissi2019,li2024pino}; neural operators such as FNOs target discretization-flexible mappings between function spaces \citep{li2020fno}. A separate line of work studies pretraining, foundation models, and in-context adaptation across PDE families \citep{subramanian2023,herde2024,mccabe2024,yang2023}. Our focus is not a new architecture but controlled attribution: whether these forms of structure improve prediction beyond a strong, equally tuned from-scratch baseline, and which regimes determine success or failure. We also connect the under-resolution failure of nonlinear residuals to the classical spectral-truncation/closure phenomenon; the full derivation and supporting experiments are in Appendix~\ref{app:residual}.

\section{Conclusion}
Known physics is not uniformly beneficial to neural PDE models. Under equal tuning budget, a known-equation residual outperforms generic regularization on resolved problems, and its advantage remains capacity-robust for nonlinear Burgers, KdV, and Allen--Cahn while collapsing toward or below parity for linear controls. The same residual becomes harmful when nonlinear dynamics are under-resolved because the coarse fields no longer satisfy the naive governing equation. Cross-family transfer, meanwhile, does not beat the strong from-scratch baseline in the regime studied, and source relevance alone does not predict successful transfer. Resolution-invariant operator behavior also does not guarantee geometry generalization, where locality can be the better inductive bias. The resulting picture is specific rather than universal: known equations help when they provide valid, non-redundant constraints; otherwise, additional structure can be neutral or harmful.

\FloatBarrier

\bibliography{references}
\bibliographystyle{plainnat}
\appendix
\section{PDE families and data generation}
\label{app:families}
The five PDE families are defined in Eq.~\ref{eq:pde_families}. This appendix records the full numerical and data-generation specification. All 1D families are solved on a periodic domain $[0,2\pi)$ with $N=64$ Fourier modes by a dealiased (2/3-rule) exponential time-differencing fourth-order Runge--Kutta (ETDRK4) integrator. Initial conditions are band-limited random fields (lowest $n_{\mathrm{modes}}$ Fourier modes excited, amplitude $a$). Snapshots are stored every $n_{\mathrm{sub}}$ substeps of size $\Delta t_{\mathrm{int}}$; the model's one-step interval is $\Delta t=n_{\mathrm{sub}}\Delta t_{\mathrm{int}}$.

\emph{Headline and capacity regime} (Fig.~\ref{fig:headline}, and the closure verification of App.~\ref{app:residual}): fixed coefficients --- heat $\nu=0.05$; adv-diff $c=1.0,\ \nu=0.03$; Burgers $\nu=0.02$; KdV $\delta=0.01$; Allen--Cahn $\varepsilon=0.002$  with substep $(\Delta t_{\mathrm{int}},n_{\mathrm{sub}})=(2\!\times\!10^{-4},25)$ for KdV and $(10^{-3},5)$ for the other four, giving a common snapshot stride $\Delta t=5\!\times\!10^{-3}$; $L=40$ stored snapshots, $64$ train / $32$ test trajectories. \emph{Transfer and relevance regime} (Figs.~\ref{fig:conditioning},~\ref{fig:relevance}): each trajectory draws its coefficient uniformly from heat $\nu\in[0.02,0.08]$; adv-diff $c\in[0.5,1.5],\ \nu\in[0.01,0.05]$; Burgers $\nu\in[0.01,0.05]$; KdV $\delta\in[0.002,0.02]$; Allen--Cahn $\varepsilon\in[0.0005,0.005]$, with $(\Delta t_{\mathrm{int}},n_{\mathrm{sub}})$ of heat $(2\!\times\!10^{-3},25)$, adv-diff and Burgers $(10^{-3},50)$, KdV $(2\!\times\!10^{-4},250)$, Allen--Cahn $(10^{-3},50)$. Both regimes use IC amplitude $a=1.0$ ($0.3$ for Allen--Cahn) and IC modes $=4$. The 2D geometry system used in the geometry-generalization experiment of Sec.~3.7 is $u_t + c\cdot\nabla u = \nu\Delta u$ on a $64\times64$ periodic grid with $\nu=2\!\times\!10^{-3}$, $c=(0.5,0.3)$, three random axis-aligned rectangular obstacles imposed as Dirichlet holes (field forced to zero inside, re-masked each step), explicit upwind-advection / 5-point-Laplacian finite differences, $\Delta t=4\!\times\!10^{-3}$, $12$ substeps, $L=20$ snapshots; each trajectory draws an independent obstacle layout.

\section{Physics-residual formulas}
\label{app:residual}
The residual loss is defined in Eq.~\ref{eq:physics_residual}. Spatial derivatives are spectral: $\partial_x^{(m)} u = \mathcal{F}^{-1}\!\big[(ik)^m \hat u\big]$. The right-hand side $F(u)$ of $u_t=F(u)$ is, per family,
\begin{align*}
F_{\text{heat}} &= \nu u_{xx}, & F_{\text{adv-diff}} &= -c u_x + \nu u_{xx}, & F_{\text{Burgers}} &= -u u_x + \nu u_{xx},\\
F_{\text{KdV}} &= -u u_x - \delta u_{xxx}, & F_{\text{Allen--Cahn}} &= \varepsilon u_{xx} + u - u^3. &&
\end{align*}
\textbf{Resolved-grid residual vs.\ coarse-grid closure.} Equation~\ref{eq:physics_residual} is exact only when the evaluation grid resolves the dynamics: the spectral derivatives in $F$ are accurate, and $r$ of the true trajectory is at the discretization floor (ground-truth $\langle r^2\rangle^{1/2}$: heat (lin.) $1.5e-11$; adv-diff (lin.) $4.8e-10$; Burgers $9.8e-06$; KdV $4.7e-06$; Allen--Cahn $1.4e-11$). On a \emph{coarsened} grid the same expression is no longer satisfied by the true coarse field coarse-graining a nonlinear operator introduces unclosed sub-grid terms, so $F$ evaluated on coarse fields omits a closure term and enforcing $r=0$ injects that omission as a systematic error (Sec.~3.4). We therefore distinguish the \emph{resolved-grid residual} we use (a correct constraint) from a \emph{coarse-grid closure} (which $r$ is not): the headline is scoped to the resolved regime, and the resolution sweep maps exactly where $r$ ceases to be a valid constraint.

\subsection{Consistency of the discrete constraint: a closure dichotomy}
\label{app:closure}

\textbf{Relation to classical spectral-truncation analysis.} The closure term $\tau_N$ formalized below is not new: it is the aliasing/closure commutator of spectral methods. That a Fourier--Galerkin truncation of a nonlinear conservation law need not satisfy the untruncated dynamicsbecause the quadratic interaction couples resolved and unresolved modes is the classical obstruction analyzed by \citet{tadmor1989}, who proved that naive Fourier methods can fail to converge for nonlinear conservation laws and introduced spectral viscosity as a remedy; $\tau_N$ is the same resolved unresolved product treated under dealiasing in the standard spectral-methods references \citep{canuto2006,hesthaven2007}, and is identical in form to the sub-grid stress of large-eddy simulation \citep{sagaut2006}. Our contribution is not this identity but its \emph{application to a learned penalty}: enforcing the discrete residual on a neural surrogate imports $\tau_N$ as an irreducible bias, and the linear/nonlinear dichotomy of $\tau_N$ becomes a diagnostic for when a known-equation residual is a valid training constraint. The following makes this precise.

\textbf{Setup.} On the periodic torus, let $P_N$ be the Fourier truncation onto the resolved modes $\{|k|\le N/2\}$ and $Q_N=I-P_N$, so $\|Q_N u\|_{L^2}^2=\sum_{|k|>N/2}|\hat u_k|^2$ is the unresolved energy. Write $F=L+B$ with $L$ the constant-coefficient linear part (a Fourier multiplier) and $B$ the nonlinear part ($B(u)=-\tfrac12\partial_x(u^2)$ for Burgers and KdV; $B(u)=-u^3$ for Allen--Cahn; $B\equiv0$ for the linear families). The dealiased ($2/3$-rule) Fourier--Galerkin operator is $F_N=P_N F P_N$, and the closure term (sub-grid commutator) is $\tau_N(u):=F_N(P_N u)-P_N F(u)$.

\begin{proposition}[Consistency dichotomy]\label{prop:closure}
Let $u\in H^s$, $s>\tfrac12$ (so $H^s\hookrightarrow L^\infty$ is a Banach algebra; $C_s$ denotes the embedding constant). Then, uniformly in $N$:
\emph{(i) Linear consistency.} If $B\equiv0$, then $\tau_N(u)=0$ for every $u$ and every $N$: a Fourier multiplier commutes with truncation, $P_N L P_N(P_N u)=P_N L u=P_N F(u)$.
\emph{(ii) Nonlinear closure.} For the quadratic flux $B(u)=-\tfrac12\partial_x(u^2)$,
\[
\tau_N(u)=-\tfrac12\,\partial_x P_N\big[(Q_N u)(u+P_N u)\big],
\]
the resolved--unresolved interaction, which vanishes iff $Q_N u=0$ and satisfies
\[
\|\tau_N(u)\|_{H^{-1}}\;\le\;C_s\,\|Q_N u\|_{L^2}\,\|u\|_{H^s},
\]
with $C_s$ independent of $N$ and $u$. For the cubic reaction $B(u)=-u^3$ the analogue holds in $L^2$ with a quadratic field factor, $\|\tau_N(u)\|_{L^2}\le C_s\|Q_N u\|_{L^2}\|u\|_{H^s}^2$.
\end{proposition}

\smallskip\noindent\emph{Proof sketch.}\ (i) $L$ is diagonal in Fourier space, hence commutes with $P_N$. (ii) Since $\partial_x,P_N$ are Fourier multipliers, $P_N F(u)=-\tfrac12\partial_x P_N(u^2)$ and (by dealiasing) $F_N(P_N u)=-\tfrac12\partial_x P_N((P_N u)^2)$; subtracting and using $(P_N u)^2-u^2=-(Q_N u)(u+P_N u)$ gives the identity. For the bound, $\partial_x\colon L^2\to H^{-1}$ and $P_N\colon L^2\to L^2$ are contractions, so $\|\tau_N\|_{H^{-1}}\le\tfrac12\|(Q_N u)(u+P_N u)\|_{L^2}\le\tfrac12\|Q_N u\|_{L^2}\|u+P_N u\|_{L^\infty}$, and $\|u+P_N u\|_{L^\infty}\le2C_s\|u\|_{H^s}$ via $\|P_N u\|_{H^s}\le\|u\|_{H^s}$ and $H^s\hookrightarrow L^\infty$. The cubic case uses $a^3-b^3=(a-b)(a^2+ab+b^2)$ with $a=P_N u,\,b=u$.\hfill$\square$\smallskip

\begin{corollary}[The enforced constraint is biased, irreducibly by data or capacity]\label{cor:bias}
Let $u(t)$ solve $u_t=F(u)$ and let $r_N$ be the one-step Crank--Nicolson residual of the projected trajectory $P_N u$. Then $r_N=-P_N\tau_N(\bar u)+O(\Delta t^2)$ for a time-average $\bar u$, so the true projected solution violates $r_N=0$ by $\|P_N\tau_N\|=O(\|Q_N u\|_{L^2})$ exactly zero for linear $F$, nonzero under-resolution for nonlinear $F$. Because this is a property of the constraint (the true target fails it), the induced bias of a model trained to enforce $r_N=0$ cannot be reduced by more data or capacity, only by resolving the dynamics, $\|Q_N u\|_{L^2}\to0$.
\end{corollary}

\textbf{Refinement rate (verified).} For $u\in H^\sigma$, $\|Q_N u\|_{L^2}\le(N/2)^{-\sigma}\|u\|_{H^\sigma}$, so $\|\tau_N\|_{H^{-1}}\le C_s(N/2)^{-\sigma}\|u\|_{H^\sigma}\|u\|_{H^s}$: the closure term vanishes under refinement at the rate of the solution's spectral regularity $\sigma$. We verify both the bound and the rate directly (Fig.~\ref{fig:closurerate}). Computing the dealiased $\tau_N$ on a fine reference grid, the linear families sit at machine epsilon (confirming $\tau_N\equiv0$), while for the nonlinear families $\|\tau_N\|$ is proportional to $\|Q_N u\|$ with fitted slope $0.99$--$1.01$ in the bound's norm ($H^{-1}$ for the quadratic-flux families, $L^2$ for the cubic reaction), and decays at the tail rate (Burgers $4.3$ vs.\ $\sigma=4.3$; KdV $3.7$ vs.\ $3.8$). At the viscous coefficients used these solutions are smooth ($\sigma\approx4$); in the inviscid, shock-forming limit $\sigma$ would fall toward the regularity of a jump, sharpening the resolution boundary a regime outside the coefficients studied here.

\begin{figure}[tb]
\centering
\includegraphics[width=\linewidth]{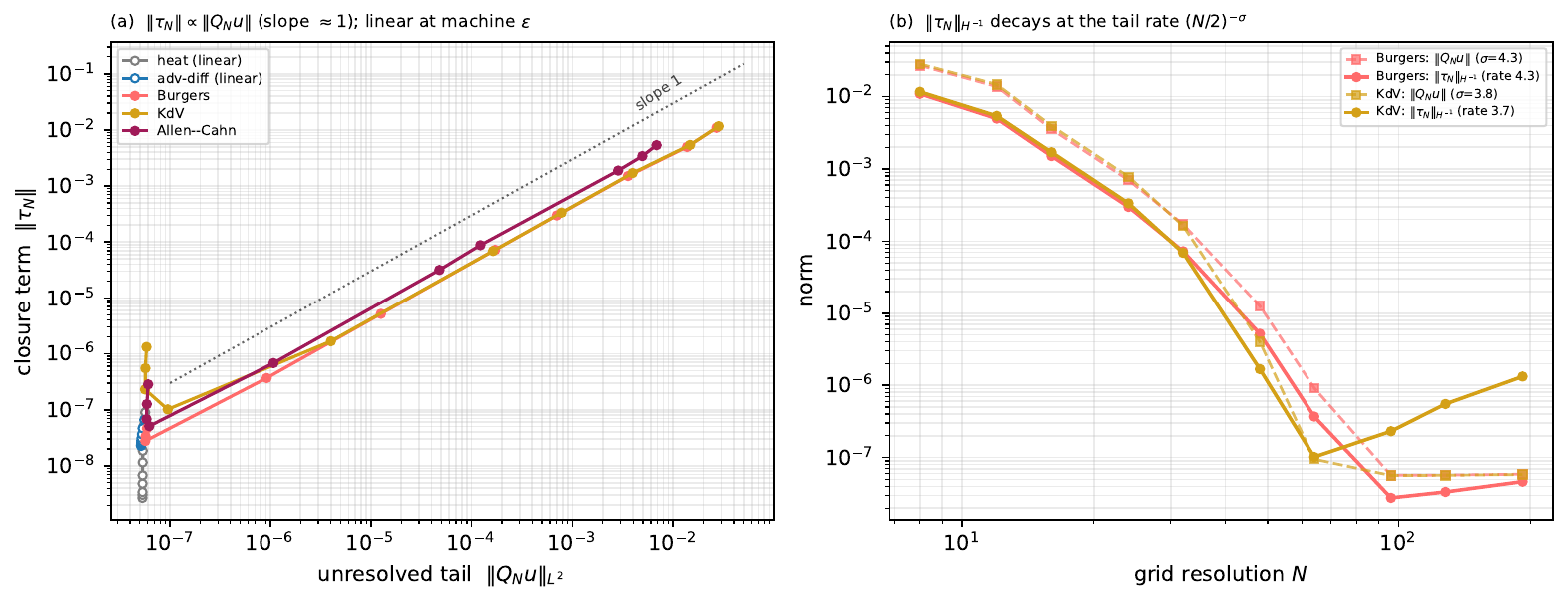}
\caption{\textbf{Direct verification of the closure bound and its refinement rate (Prop.~\ref{prop:closure}).} The dealiased Galerkin closure term $\tau_N=P_N[F(P_N u)-F(u)]$ computed on a fine reference grid. (a) $\|\tau_N\|$ versus the unresolved tail $\|Q_N u\|_{L^2}$: the nonlinear families fall on the slope-$1$ line predicted by $\|\tau_N\|\le C\|Q_N u\|\,\|u\|_s$ (fitted slopes $0.99$, $0.98$, $1.01$ in the bound's norm), while the linear families collapse to a machine-$\epsilon$ cluster, $\tau_N\equiv0$. (b) For Burgers and KdV, $\|\tau_N\|_{H^{-1}}$ (solid) and $\|Q_N u\|$ (dashed) decay in parallel under refinement, the closure rate matching the fitted spectral-regularity exponent $\sigma$.}
\label{fig:closurerate}
\end{figure}

\section{Architectures, hyperparameters, selection, and cost}
\label{app:hparams}
\textbf{Architectures.} 1D: Fourier Neural Operator (FNO) in residual form, $16$ retained modes, width $32$, $3$ spectral-convolution layers, GELU; the from-scratch UNet baseline (Fig.~\ref{fig:conditioning}) matches parameter budget with a periodic-convolution encoder--decoder. 2D: FNO2d ($12$ modes per axis, width $32$, $4$ layers) and a matched CNN2d ($5\times5$ periodic convolutions, width $32$, $4$ layers), both mask-conditioned (field, obstacle-mask) input channels. Optimization is Adam with a warmup--cosine-decay schedule (peak lr $10^{-3}$, $10\%$ warmup, $5\%$ floor) and global-norm gradient clipping at $1.0$; batch $64$.

\begin{center}\footnotesize
\begin{tabular}{lcccccccc}
\hline
Experiment (figure) & $N$ & $K$ & width & modes & depth & epochs & seeds & training \\
\hline
Headline / linear controls (\ref{fig:headline}a) & 64 & 1 & 32 & 16 & 3 & 150 & 5 & one-step \\
Capacity sweep (\ref{fig:headline}b) & 64 & 1 & $\{8..128\}$ & 16 & 3 & 150 & 5 & one-step \\
Pushforward ladder (\ref{fig:pushforward}) & 64 & 4 & 32 & 16 & 3 & 200 & 5 & 1- vs.\ $K$-step \\
Conditioning / arch (\ref{fig:conditioning}) & 64 & 5 & 32 & 16 & 3 & 150/200 & 5 & $K$-step \\
Relevance gradient (\ref{fig:relevance}) & 64 & 5 & 32 & 16 & 3 & 150/150 & 5$\times$3 & $K$-step \\
Resolution invariance (\ref{fig:resolution}) & 128$\to$256/512 & 4 & 32 & 16 & 3 & 300 & 5 & $K$-step \\
Geometry (\ref{fig:geometry},\ref{fig:geometry_soln}) & 64 & 4 & 32 & 12 & 4 & 200 & 5 & $K$-step \\
\hline
\end{tabular}
\end{center}

\textbf{Penalty and regularizer grids.} The physics arm sweeps $\lambda\in\{0.01,\,0.1,\,1.0\}$. The generic-regularizer baseline is the \emph{best of} a matched-budget grid: data-only (no regularizer), weight decay $\in\{10^{-4},10^{-3}\}$, and input-noise augmentation (std $10^{-2}$) four configurations, the same cardinality as is afforded the physics arm.

\textbf{Selection protocol (consistent with Section~2.4).} For each arm we train every configuration in its grid, select the one with the lowest error on a held-out \emph{validation} ensemble (a third trajectory set, seed-disjoint from training and test), and report that configuration's error on the \emph{test} ensemble; the same rule is applied to the physics and the generic arm. Because selection never sees the test set, both the absolute errors and the ratio $E_{\mathrm{generic}}/E_{\mathrm{physics}}$ are unbiased. Splits are seed-controlled: training seed $10$, validation seed $55$, test seed $99$.

\textbf{Training-cost overhead.} Measured on dual RTX~PRO~6000 (Blackwell) GPUs at the headline config ($N=64$, width $32$, batch $64$, $\lambda=0.1$; $200$ timed steps after JIT warmup, device-synchronized), the physics penalty adds $11$--$23\%$ wall-clock per training step relative to a data-only step. The overhead tracks the number of spectral derivatives in $F$: $+11$--$12\%$ for the single-derivative families (heat, Allen-Cahn) and $+20$--$23\%$ for the two- and three-derivative families (advection diffusion, Burgers, KdV). In absolute terms this is $0.02$--$0.05$\,ms per step one $O(N\log N)$ evaluation of $F$ (FFT-based spectral derivatives) plus the Crank--Nicolson combination negligible beside the multi-layer operator forward/backward pass. A timing harness reproducing these figures is released with the code.

\section{Pre-registration record}
\label{app:prereg}
Before the deciding flip run we registered four hypotheses and falsification criteria (full record: \texttt{PGWM\_flip\_preregistration.md}).
\begin{itemize}
\item \textbf{H-flip-1 (inert dense end):} the physics advantage vanishes at full supervision ($\rho=1$). \emph{Decision:} reject if the $\rho=1$ interval excludes parity for any nonlinear family. \textbf{Outcome: falsified} physics wins at $\rho=1$ for all three nonlinear families beyond their intervals.
\item \textbf{H-flip-2 (monotone sparsity trend):} the advantage increases as $\rho$ decreases. \textbf{Outcome: supported}, but reinterpreted this is the ordinary ``physics helps more when data is scarce,'' not a gating threshold.
\item \textbf{H-flip-3 (bandwidth crossover):} a crossover from helpful to inert occurs near $\rho^\star\approx 2m/N$ set by the solution bandwidth $m$. \emph{Decision:} confirm only if an inert dense regime exists above $\rho^\star$. \textbf{Outcome: falsified}  with no inert dense end (H-flip-1), there is no crossover for $\rho^\star$ to mark. Measured $\rho^\star=0.156$ ($m=5$, $N=64$) for all families.
\item \textbf{H-flip-4 (bandwidth separation across families):} families separate by bulk spectral bandwidth $m_{99}$. \textbf{Outcome: untestable} on this suite all five families share $m_{99}=5$, so a bandwidth-tracking law cannot be distinguished here.
\end{itemize}
We report all four as registered, including the two falsifications, and do not re-cast the headline as a sparsity law.

\section{Per-condition results with confidence intervals}
\label{app:supp}
Held-out one-step relative $L_2$ error for the best generic-regularized arm ($E_{\mathrm{generic}}$) and the best physics arm ($E_{\mathrm{physics}}$), their ratio, and the difference $\Delta=E_{\mathrm{generic}}-E_{\mathrm{physics}}$ with its $95\%$ bootstrap CI (percentile, $500$ resamples, $5$ seeds), for every family and observation density $\rho$ ($N=64$). Every $\Delta$ interval excludes zero, so the ratio $>1$ is significant at all conditions.

\begin{center}\footnotesize
\begin{longtable}{llccccc}
\hline
family & $\rho$ & $E_{\mathrm{generic}}$ & $E_{\mathrm{physics}}$ & ratio & $\Delta$ & $95\%$ CI on $\Delta$ \\
\hline
\endhead
heat (lin.) & 0.03 & $1.92\!\times\!10^{-6}$ & $3.54\!\times\!10^{-7}$ & 5.41 & $1.56\!\times\!10^{-6}$ & [6.1e-07, 2.6e-06] \\
heat (lin.) & 0.06 & $1.76\!\times\!10^{-6}$ & $3.55\!\times\!10^{-7}$ & 4.96 & $1.41\!\times\!10^{-6}$ & [5.2e-07, 2.4e-06] \\
heat (lin.) & 0.12 & $1.64\!\times\!10^{-6}$ & $3.53\!\times\!10^{-7}$ & 4.65 & $1.29\!\times\!10^{-6}$ & [4.8e-07, 2.2e-06] \\
heat (lin.) & 0.25 & $1.65\!\times\!10^{-6}$ & $3.51\!\times\!10^{-7}$ & 4.69 & $1.30\!\times\!10^{-6}$ & [4.5e-07, 2.2e-06] \\
heat (lin.) & 0.5 & $1.62\!\times\!10^{-6}$ & $3.51\!\times\!10^{-7}$ & 4.63 & $1.27\!\times\!10^{-6}$ & [4.3e-07, 2.2e-06] \\
heat (lin.) & 1.0 & $1.63\!\times\!10^{-6}$ & $3.55\!\times\!10^{-7}$ & 4.58 & $1.27\!\times\!10^{-6}$ & [4.3e-07, 2.2e-06] \\
adv-diff (lin.) & 0.03 & $4.33\!\times\!10^{-6}$ & $1.72\!\times\!10^{-6}$ & 2.52 & $2.62\!\times\!10^{-6}$ & [2.1e-06, 3.3e-06] \\
adv-diff (lin.) & 0.06 & $3.42\!\times\!10^{-6}$ & $1.72\!\times\!10^{-6}$ & 1.99 & $1.70\!\times\!10^{-6}$ & [1.1e-06, 2.3e-06] \\
adv-diff (lin.) & 0.12 & $2.98\!\times\!10^{-6}$ & $1.71\!\times\!10^{-6}$ & 1.74 & $1.26\!\times\!10^{-6}$ & [8.2e-07, 1.7e-06] \\
adv-diff (lin.) & 0.25 & $2.60\!\times\!10^{-6}$ & $1.72\!\times\!10^{-6}$ & 1.52 & $8.87\!\times\!10^{-7}$ & [5.8e-07, 1.3e-06] \\
adv-diff (lin.) & 0.5 & $2.49\!\times\!10^{-6}$ & $1.72\!\times\!10^{-6}$ & 1.45 & $7.73\!\times\!10^{-7}$ & [4.4e-07, 1.2e-06] \\
adv-diff (lin.) & 1.0 & $2.47\!\times\!10^{-6}$ & $1.73\!\times\!10^{-6}$ & 1.43 & $7.39\!\times\!10^{-7}$ & [4.0e-07, 1.1e-06] \\
Burgers & 0.03 & $1.54\!\times\!10^{-4}$ & $1.73\!\times\!10^{-5}$ & 8.90 & $1.37\!\times\!10^{-4}$ & [1.3e-04, 1.5e-04] \\
Burgers & 0.06 & $1.09\!\times\!10^{-4}$ & $1.69\!\times\!10^{-5}$ & 6.47 & $9.24\!\times\!10^{-5}$ & [9.0e-05, 9.5e-05] \\
Burgers & 0.12 & $8.64\!\times\!10^{-5}$ & $1.77\!\times\!10^{-5}$ & 4.88 & $6.87\!\times\!10^{-5}$ & [6.5e-05, 7.2e-05] \\
Burgers & 0.25 & $7.34\!\times\!10^{-5}$ & $1.72\!\times\!10^{-5}$ & 4.27 & $5.62\!\times\!10^{-5}$ & [5.4e-05, 5.9e-05] \\
Burgers & 0.5 & $6.45\!\times\!10^{-5}$ & $1.73\!\times\!10^{-5}$ & 3.73 & $4.72\!\times\!10^{-5}$ & [4.4e-05, 5.1e-05] \\
Burgers & 1.0 & $5.85\!\times\!10^{-5}$ & $1.71\!\times\!10^{-5}$ & 3.41 & $4.13\!\times\!10^{-5}$ & [3.9e-05, 4.3e-05] \\
KdV & 0.03 & $1.87\!\times\!10^{-4}$ & $1.97\!\times\!10^{-5}$ & 9.48 & $1.67\!\times\!10^{-4}$ & [1.5e-04, 1.8e-04] \\
KdV & 0.06 & $1.23\!\times\!10^{-4}$ & $1.98\!\times\!10^{-5}$ & 6.21 & $1.03\!\times\!10^{-4}$ & [9.9e-05, 1.1e-04] \\
KdV & 0.12 & $9.82\!\times\!10^{-5}$ & $2.00\!\times\!10^{-5}$ & 4.91 & $7.82\!\times\!10^{-5}$ & [7.5e-05, 8.1e-05] \\
KdV & 0.25 & $7.90\!\times\!10^{-5}$ & $2.06\!\times\!10^{-5}$ & 3.84 & $5.84\!\times\!10^{-5}$ & [5.4e-05, 6.2e-05] \\
KdV & 0.5 & $7.06\!\times\!10^{-5}$ & $2.05\!\times\!10^{-5}$ & 3.44 & $5.01\!\times\!10^{-5}$ & [4.6e-05, 5.4e-05] \\
KdV & 1.0 & $6.53\!\times\!10^{-5}$ & $2.05\!\times\!10^{-5}$ & 3.19 & $4.48\!\times\!10^{-5}$ & [4.2e-05, 4.7e-05] \\
Allen--Cahn & 0.03 & $2.15\!\times\!10^{-4}$ & $2.65\!\times\!10^{-5}$ & 8.11 & $1.88\!\times\!10^{-4}$ & [1.6e-04, 2.1e-04] \\
Allen--Cahn & 0.06 & $9.93\!\times\!10^{-5}$ & $2.65\!\times\!10^{-5}$ & 3.75 & $7.28\!\times\!10^{-5}$ & [6.0e-05, 8.9e-05] \\
Allen--Cahn & 0.12 & $6.45\!\times\!10^{-5}$ & $2.65\!\times\!10^{-5}$ & 2.44 & $3.81\!\times\!10^{-5}$ & [3.1e-05, 4.7e-05] \\
Allen--Cahn & 0.25 & $5.64\!\times\!10^{-5}$ & $2.66\!\times\!10^{-5}$ & 2.13 & $2.99\!\times\!10^{-5}$ & [2.6e-05, 3.4e-05] \\
Allen--Cahn & 0.5 & $5.21\!\times\!10^{-5}$ & $2.66\!\times\!10^{-5}$ & 1.96 & $2.55\!\times\!10^{-5}$ & [2.2e-05, 2.8e-05] \\
Allen--Cahn & 1.0 & $5.01\!\times\!10^{-5}$ & $2.65\!\times\!10^{-5}$ & 1.89 & $2.36\!\times\!10^{-5}$ & [2.2e-05, 2.6e-05] \\
\hline
\end{longtable}
\end{center}

\section{Repository layout and provenance}
\label{app:repo}
Code spans two directories with some overlapping filenames; each figure script states its source directory (\texttt{sys.path}) to avoid importing a mismatched module (e.g.\ \texttt{model.py} and \texttt{physics.py} differ between them). Phase-2 (flip/headline, capacity, conditioning, relevance): \texttt{world\_model\_v2/}. Phase-1 (pushforward ladder, resolution invariance, geometry): \texttt{World\_Model/}. Figure scripts and a master driver that reproduce all panels from the run JSONs are released with the paper.

\section{Multi-target relevance replication}
\label{app:relmulti}
The relevance gradient is run for three held-out targets---Burgers, KdV, and Allen--Cahn---each pretraining four source conditions and finetuning over $M\times\text{reps}\times\text{seeds}$ ($5\times3$ at each $M$). Table~\ref{tab:relmulti} reports the data-adequate endpoint ($M=16$); full per-$M$ curves are Fig.~\ref{fig:relevance}. The refutation reproduces on both new targets: the most relevant source (\texttt{pde\_related}) is significantly worse than from-scratch, the rich non-physical operator (\texttt{struct\_random}) is inert (KdV) or marginally helpful (Allen Cahn), and structureless sources (\texttt{noise\_map}, \texttt{pure\_noise}) are destructive. The single non-monotone detail is a transient KdV head start from relevant-physics pretraining at $M=1$ ($0.525$ $[0.498,0.550]$ vs from-scratch $0.632$ $[0.567,0.699]$) that reverses by $M=4$.

\begin{table}[t]
\centering
\caption{Held-out rollout rel-$L_2$ at $M=16$ (bootstrap $95\%$ CI, $5\times3$ runs). ``Related'' denotes \texttt{pde\_related}, ``Rich'' denotes \texttt{struct\_random}, ``No map'' denotes \texttt{noise\_map}, and ``Noise'' denotes \texttt{pure\_noise}.}
\label{tab:relmulti}
\small
\setlength{\tabcolsep}{4pt}
\begin{tabular}{llccccc}
\hline
Target & Metric & Related & Rich & No map & Noise & Scratch \\
\hline
KdV & Mean & 0.473 & 0.428 & 0.797 & 0.994 & 0.427 \\
& 95\% CI & [.464,.484] & [.412,.442] & [.758,.846] & [.941,1.05] & [.410,.443] \\
Allen--Cahn & Mean & 0.185 & 0.048 & 1.184 & 1.370 & 0.055 \\
& 95\% CI & [.172,.198] & [.046,.050] & [1.04,1.32] & [1.20,1.54] & [.054,.056] \\
\hline
\end{tabular}
\end{table}

\end{document}